\documentclass{article}

\PassOptionsToPackage{numbers, compress}{natbib}

\usepackage[preprint]{neurips_2026}

\usepackage{caption}
\usepackage[utf8]{inputenc} 
\usepackage[T1]{fontenc}    
\usepackage{hyperref}       
\usepackage{url}            
\usepackage{booktabs}       
\usepackage{amsfonts}       
\usepackage{nicefrac}       
\usepackage{microtype}      
\usepackage{xcolor}         
\usepackage{enumitem}
\usepackage{etoolbox}
\usepackage{booktabs,makecell,tabularx} 
\usepackage{resizegather}
\usepackage{multirow}
\usepackage{amssymb}
\usepackage{amsmath}
\usepackage{graphicx}
\usepackage{tcolorbox}
\usepackage{subcaption}
\usepackage{colortbl}
\usepackage{listings}
\usepackage{makecell}
\usepackage{wrapfig} 
\usepackage{xspace}
\usepackage{tcolorbox}
\usepackage{longtable}
\usepackage{amsthm}
\usepackage{amsmath}
\usepackage[export]{adjustbox}

\newtheorem{proposition}{Proposition}

\newcommand{\M}{CAVI\xspace}

\title{Through the Lens of Character: Resolving Modality-Role Interference in Multimodal Role-Playing Agent}
  
\author{Yihong Tang$^{1,2}$, Kehai Chen$^{1,}$\thanks{Corresponding author.}~~, Xuefeng Bai$^{1}$,  Min Zhang$^1$
\\
        $^1$Harbin Institute of Technology, Shenzhen, China \\
        $^2$Shenzhen Loop Area Institute (SLAI), Shenzhen, China \\
        \texttt{\{toyhom@stu.hit.edu.cn,  chenkehai@hit.edu.cn\}}
}

\begin{document}

\maketitle

\begin{abstract}
The advancement of Multimodal Large Language Models (MLLMs) has expanded Role-Playing Agents (RPAs) into visually grounded environments. However, human vision is inherently subjective and identity-driven, whereas existing MLLMs extract objective, character-agnostic features for general tasks. In RPAs, this generic visual noise overpowers fragile character traits, causing Modality-Role Interference (MRI), where agents struggle to integrate visual grounding and character consistency. To address this, we introduce the training-free Character-Aware Visual Intervention (CAVI) framework, enabling agents to perceive the world through the lens of character. CAVI systematically targets MRI: macroscopically, Character-Guided Token Pruning (CTP) restricts the visual receptive field to role-relevant entities; microscopically, Orthogonal Feature Modulation (OFM) projects tokens onto a character-context subspace to extract aligned facts; and during decoding, Modality-Adaptive Role Steering (MARS) dynamically optimizes steering intensity based on visual reliance. Extensive experiments show CAVI effectively alleviates MRI, significantly enhancing character-consistent multimodal interactions.
\end{abstract}

\section{Introduction}\label{sec: intro}

With the advancement of Multimodal Large Language Models (MLLMs), Role-Playing Agents (RPAs)~\citep{chen2024oscarsaitheatersurvey} have naturally extended their boundaries to visually grounded environments. Many multimodal RPAs explicitly integrate rich role-playing profiles, diverse visual contexts, and the inherent generalization capabilities of MLLMs~\citep{dai2025mmrole, zhang2025video2roleplay}, yielding promising results.

Despite their success, existing methods often overlook the subjective nature of character perception. In the real world, human visual perception is heavily conditioned by identity. As shown in Figure~\ref{fig: main} (a), when observing a forest, a botanist naturally focuses on vegetation health, whereas a lumberjack evaluates timber volume. However, existing MLLMs employ vision encoders that extract objective, character-agnostic features tailored primarily for general visual tasks. 

In role-playing scenarios, the massive influx of these objective visual details structurally clashes with the persona, a cognitive mismatch we formalize as \textbf{Modality-Role Interference (MRI)}. As illustrated in Figure~\ref{fig: main} (b), when given a fixed image and question but different role prompts, the base model retrieves highly similar visual evidence. This indicates that rather than organically adopting the character's unique perspective, the model's visual perception is only weakly modulated by identity. Consequently, the agent struggles to generate responses that are both visually grounded and faithfully aligned with the persona.
While supervised fine-tuning could mitigate MRI, it brings prohibitive annotation costs by requiring customized image-query-response triplets for different characters viewing the same image.

To this end, this paper introduces a training-free Character-Aware Visual Intervention (\textbf{CAVI}) framework, designed to imbue MLLMs with the capacity to perceive the world \textbf{\textit{Through the Lens of Character}}. From an information-theoretic perspective, CAVI functions as a deterministic cognitive bottleneck that synergizes macro-level, micro-level, and decoding-level interventions. \textit{At the macro-level}, Character-Guided Token Pruning (CTP) explicitly bounds the visual receptive field to retain only character-relevant entities. Subsequently, at the \textit{micro-level}, Orthogonal Feature Modulation (OFM) projects these retained tokens onto a low-rank character-context subspace, purifying role-consistent visual facts from objective noise.  During \textit{decoding level}, Modality-Adaptive Role Steering (MARS) with a contextual re-injection mechanism dynamically adjust the role-steering intensity based on the model's instantaneous visual reliance. 

Through this pipeline, the agent is able to smoothly synthesize rigorous visual grounding with vivid role portrayal.
Specifically, comprehensive experiments reveal that CAVI achieves state-of-the-art multimodal role-playing performance among comparable methods, while simultaneously indicating robust generalization to unseen characters, yielding clear mechanistic interpretability, and introducing minimal additional inference overhead.

Our contributions are summarized as follows: 
(1) We identify Modality-Role Interference (MRI) in Multimodal RPAs, highlighting the cognitive mismatch between objective MLLM visual features and subjective character perception, and propose the \textit{Through the Lens of Character} paradigm; 
(2) We introduce CAVI, an training-free framework that systematically resolves MRI via macro-level token pruning (CTP), micro-level orthogonal feature modulation (OFM), and decoding-level modality-adaptive steering (MARS); 
(3) Supported by information-theoretic analysis, extensive experiments indicate that CAVI effectively alleviates MRI, significantly enhancing the overall performance of Multimodal RPAs.

\section{Related Work}

\begin{figure}[t]
    \centering
    \begin{subfigure}[c]{0.4\linewidth}
        \centering
        \includegraphics[width=.9\linewidth]{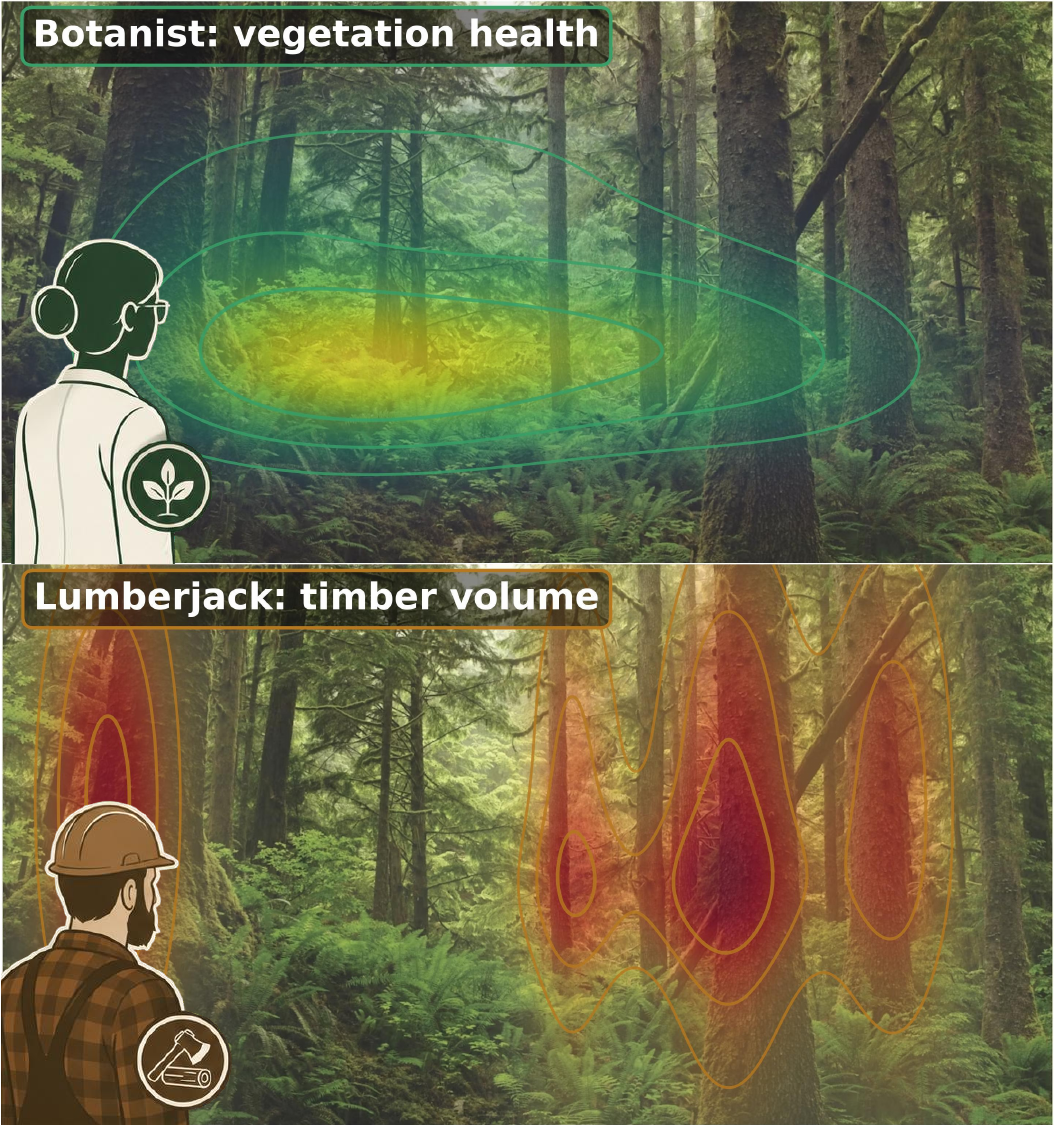}
        \caption{} 
        \label{fig:main_a}
    \end{subfigure}
    \begin{subfigure}[c]{0.5\linewidth}
        \centering
        \includegraphics[width=.9\linewidth]{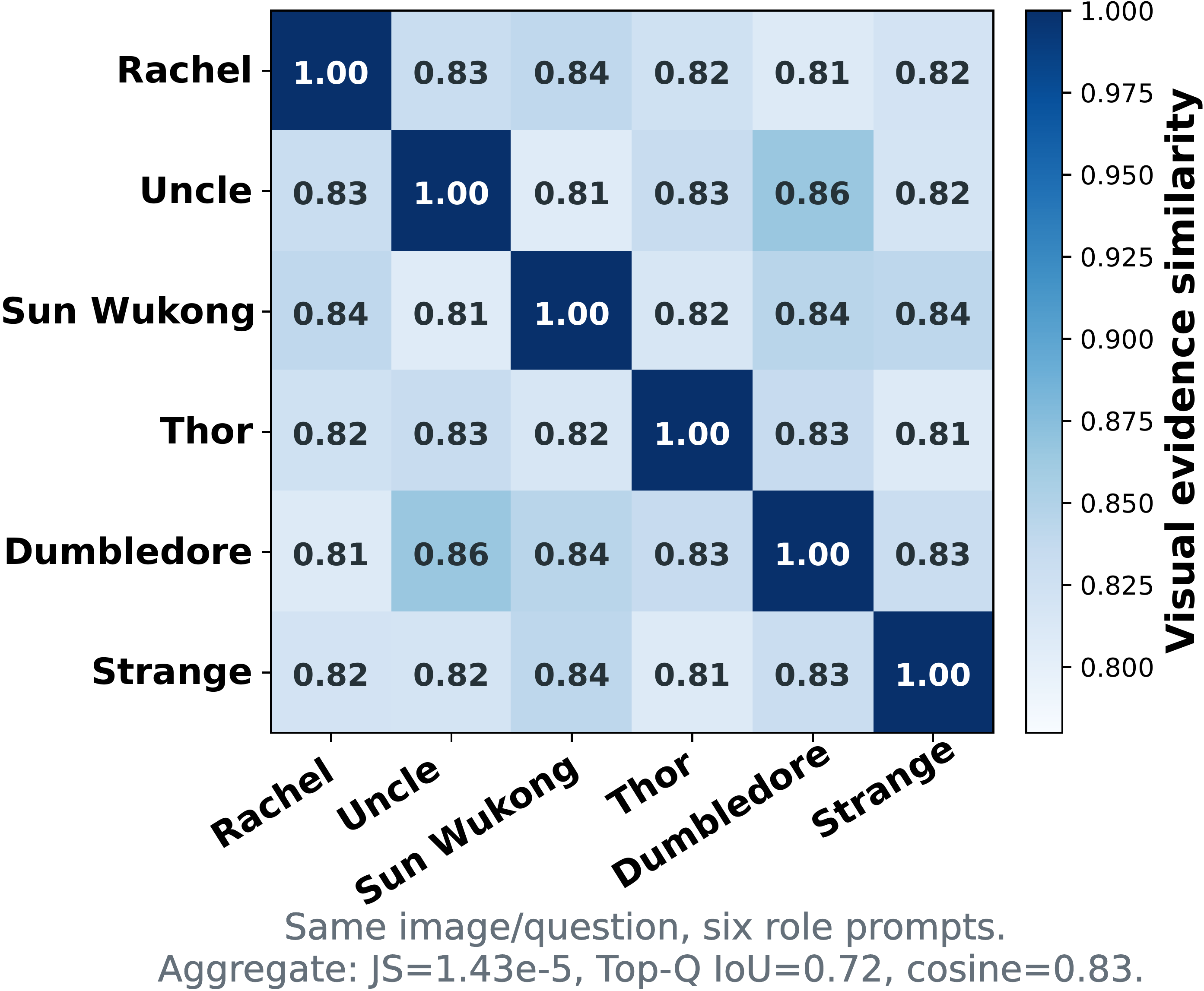}
        \caption{} 
        \label{fig:main_b}
    \end{subfigure}
    \caption{Illustration of Modality-Role Interference (MRI). \textbf{(a)} Human visual perception is subjectively conditioned by identity, where different roles (e.g., a botanist vs. a lumberjack) naturally focus on distinct visual aspects of the same environment. \textbf{(b)} Existing MLLMs retrieve highly similar visual evidence across different roles for the same image and question, indicating that character identity only weakly modulates visual perception.}
    \label{fig: main}
    \vspace{-0.3cm}
\end{figure}

\subsection{Role-Playing Agents with Large Models}

Large language models have substantially advanced text-based Role-Playing Agents (RPAs)~\citep{chen2024oscarsaitheatersurvey}. Early methods primarily rely on prompt engineering and in-context learning~\citep{tang-etal-2023-enhancing-personalized,Wei2023MultiPartyCC,Mao2023EditingPF,wang-etal-2024-incharacter}, while subsequent works improve character fidelity by fine-tuning on role-specific corpora~\citep{shao2023character,Xu2024CharacterID,zhou2024characterglm,wang-etal-2024-rolellm,yu2024neeko}. Recent studies further enhance persona consistency through memory retrieval, structured knowledge, and explicit cognitive simulation~\citep{park2023generativeagents,Sun2024IdentityDrivenHR,xu2024characterdestiny,park2024charactergpt,wang2025rolerag,wang2025coser,tang2025thinking,tang2026character}. Recently, RPAs have naturally expanded into visually grounded scenarios via Multimodal Large Language Models (MLLMs)~\citep{yin2023survey, dai2025mmrole, zhang2025video2roleplay}, exploring applications such as visual role-play attacks~\citep{ma2024visualroleplay}, embodied environment interaction~\citep{wang2024mobileagent}, and role-aware multimodal retrieval~\citep{wang2025rolerag}.
However, these approaches invariably pair text-aligned personas with generic, character-agnostic vision encoders, introducing objective visual noise that triggers Modality-Role Interference (MRI). Distinct from prior methods, our CAVI framework resolves MRI by shifting visual perception from an objective baseline to a subjective, character-aware perspective.

\subsection{Visual Hallucination Mitigation in MLLMs}

Although MLLMs have achieved strong performance on vision-language tasks, they still suffer from visual hallucination, where generated responses mention objects, attributes, relations, or events inconsistent with the given image~\citep{liu2024surveyhallucination}. Existing analyses attribute this issue to insufficient visual grounding, biased vision-language alignment, and over-reliance on parametric knowledge when visual evidence is weak~\citep{li2023evaluating,fu2023mme,wang2023amber,chen2024multiobject}. To mitigate hallucination, prior work has explored training-based alignment with preference data, counterfactual samples, or fine-grained feedback~\citep{xie2024vdpo,xiao2025detecting,liu2024cliphallucination}, post-hoc verification and revision through LURE, Woodpecker, and LogicCheckGPT~\citep{zhou2023lure,yin2023woodpecker,wu2024logiccheckgpt}, and training-free decoding-time intervention methods such as VCD, OPERA, HALC, AGLA, and attention-based diagnostics that suppress language priors, penalize unreliable attention patterns, or strengthen grounding in discriminative visual regions~\citep{leng2024mitigating,huang2023opera,chen2024halc,an2025agla,jiang2025devils,gong2024damro}. These methods improve the factual reliability of MLLMs without always requiring expensive retraining.
However, standard hallucination mitigation generally pursues objective visual faithfulness, which differs from the subjective nature of character perception. In contrast, CAVI modulates visual evidence according to the character's unique worldview, enabling responses that are both visually grounded and role-consistent.

\section{Methodology}
\label{sec:method}

\begin{figure}[t]
    \centering
    \includegraphics[width=\linewidth]{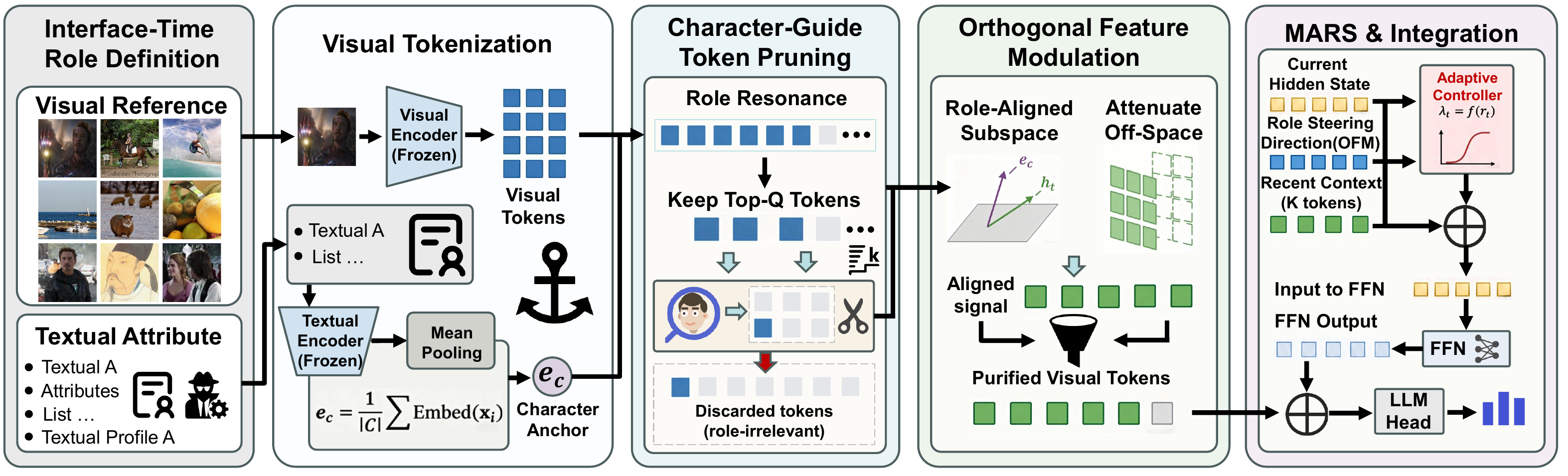}
    \caption{\textbf{Overview of Character-Aware Visual Intervention (CAVI).} CAVI resolves Modality-Role Interference via a training-free pipeline: \textbf{(1) Character Anchor} extracts a semantic role reference from textual profiles. \textbf{(2) Token Pruning (CTP)} macroscopically discards role-irrelevant visual tokens. \textbf{(3) Feature Modulation (OFM)} microscopically purifies retained tokens by projecting out orthogonal out-of-character noise. \textbf{(4) MARS \& Integration} dynamically adapts steering intensity based on visual reliance, and re-injects purified tokens into FFNs to preserve factual grounding.}
    \label{fig:main}
\end{figure}

To resolve Modality-Role Interference (MRI) without incurring prohibitive fine-tuning costs, we propose \textbf{CAVI} (\textbf{C}haracter-\textbf{A}ware \textbf{V}isual \textbf{I}ntervention), a deterministic, fully training-free framework built upon the information bottleneck principle. CAVI systematically aligns visual perception and role-playing dialogue generation through macroscopic token pruning (CTP), microscopic orthogonal feature modulation (OFM), and dynamic activation steering (MARS).

\subsection{Problem Formulation and Character Anchor}
Given a Multimodal Large Language Model (MLLM), the input sequence is $\mathcal{X} = [\mathbf{Z}_{raw}, \mathbf{X}_{sys}, \mathbf{X}_{query}]$, where $\mathbf{Z}_{raw} \in \mathbb{R}^{K \times d}$ represents objective visual tokens, and $\mathbf{X}_{sys}$ contains the textual character profile $\mathcal{C}$. To establish a stable geometric reference for the target character in the latent space, we define the \textbf{Character Anchor} $\mathbf{e}_c \in \mathbb{R}^d$ by mean-pooling the hidden states of $\mathcal{C}$ at the embedding layer:
$
\mathbf{e}_c = \frac{1}{|\mathcal{C}|} \sum_{x_i \in \mathcal{C}} \text{Embed}(x_i)
$.

\textbf{Information-Theoretic View of MRI:} In standard decoding, the dense visual representations $\mathbf{Z}_{raw}$ introduce an objective semantic subspace. From an information bottleneck (IB) perspective, an ideal conversational hidden state $h_t$ should maximize mutual information (MI) with both the factual visual grounding $\mathcal{V}_{fact}$ and the character anchor $\mathbf{e}_c$, while minimizing MI with irrelevant out-of-character background noise $\mathcal{Z}_{OOC}$:
\begin{equation}
\label{eq:ib}
\mathcal{L}_{IB} = \max I(h_t; \mathcal{V}_{fact}, \mathbf{e}_c) - \beta I(h_t; \mathcal{Z}_{OOC}).
\end{equation}
During cross-modal fusion, the massive influx of generic visual tokens forces the model to attend to $\mathcal{Z}_{OOC}$, structurally overwriting the fragile role traits encoded in $\mathbf{e}_c$.


\subsection{Macro-Level: Character-Guided Token Pruning (CTP)}
To truncate macroscopic global background noise, CTP extracts an optimal visual subset $\mathbf{Z}_{sub} \subset \mathbf{Z}_{raw}$. For each visual token $z_k \in \mathbf{Z}_{raw}$, we compute its \textit{Semantic Resonance Score} with the character anchor:
\begin{equation}
\label{eq:resonance}
\mathcal{S}_k = \frac{z_k^\top \mathbf{e}_c}{\lVert z_k \rVert_2 \lVert \mathbf{e}_c \rVert_2}, \quad \forall k \in \{1, \dots, K\}.
\end{equation}
We then apply a deterministic indicator function to construct $\mathbf{Z}_{sub} \in \mathbb{R}^{Q \times d}$ by retaining exactly the top-$Q$ tokens (the top-q proportion): $\mathbf{Z}_{sub} = \{ z_k \mid \mathbb{I}(\mathcal{S}_k \ge \tau_Q) = 1 \}$.

\subsection{Micro-Level: Orthogonal Feature Modulation (OFM)}
While CTP filters tokens globally, individual retained tokens still contain an entangled superposition of character-relevant facts and objective Out-Of-Character (OOC) noise ($\hat{z} = z^{role} + z^{noise}$). As mathematically proven in Appendix~\ref{app:prop1}, standard scalar attention mechanisms fundamentally fail to decouple this intra-token feature entanglement.

To physically extract role-consistent facts, at decoding step $t$ and layer $l$, the conversational hidden state $h_t^{(l)}$ and the anchor $\mathbf{e}_c$ span \textbf{Character-Context Subspace} $\mathcal{V}_{CC} = \text{span}(h_t^{(l)}, \mathbf{e}_c)$. We construct an orthonormal basis $\mathbf{U}_t = [u_1, u_2] \in \mathbb{R}^{d \times 2}$ using the Gram-Schmidt process:
\begin{equation}
u_1 = \frac{\mathbf{e}_c}{\|\mathbf{e}_c\|_2}, \quad v_2 = h_t^{(l)} - (h_t^{(l)\top} u_1)u_1, \quad u_2 = \frac{v_2}{\|v_2\|_2}.
\end{equation}
The orthogonal projection matrix is defined as $\mathbf{P}_{CC} = \mathbf{U}_t \mathbf{U}_t^\top \in \mathbb{R}^{d \times d}$. For each token $\hat{z} \in \mathbf{Z}_{sub}$, OFM modulates the feature by heavily attenuating the orthogonal OOC noise residual:
\begin{equation}
\tilde{z} = \mathbf{P}_{CC}\hat{z} + \gamma (\mathbf{I} - \mathbf{P}_{CC})\hat{z},
\end{equation}
where $\gamma \in [0, 1)$ strictly bounds noise leakage. Appendix~\ref{app:prop2} geometrically proves that projecting onto $\mathcal{V}_{CC}$ leverages the high dimensionality of MLLMs ($d \gg 1$) to exponentially crush the expected OOC noise variance, amplifying the intra-token Signal-to-Noise Ratio (SNR) by orders of magnitude.

\subsection{Decoding-Level: Modality-Adaptive Role Steering (MARS)}
\label{sec:mars}
During prolonged auto-regressive decoding, models inherently suffer from role drift. Meanwhile, applying a static textual steering vector indiscriminately overwrites visual grounding, causing hallucinations. MARS dynamically modulates the steering intensity based on the model's instantaneous visual reliance.

For a targeted set of intervention layers $\mathcal{L}_{steer}$, we offline extract an expected perturbation direction $\mathbf{v}_{r}^{(l)} \in \mathbb{R}^d$ using a paired calibration dataset: $\mathbf{v}_{r}^{(l)} = \frac{1}{N} \sum_{i=1}^{N} \left( h_{role, i}^{(l)} - h_{neu, i}^{(l)} \right)$, where $h_{role, i}^{(l)}$ and $h_{neu, i}^{(l)}$ denote the hidden states produced under role-playing prompts and neutral prompts, respectively.
During online decoding, we quantify the model's instantaneous visual reliance $\omega_t^{(l)}$ by aggregating the attention mass over all heads $N_H$ assigned to the visual tokens $\mathcal{I}_{vis}$:
$\omega_t^{(l)} = \frac{1}{N_H} \sum_{h=1}^{N_H} \sum_{j \in \mathcal{I}_{vis}} \mathbf{A}_{t, j}^{(h, l)} \quad \in [0, 1].$
The conversational hidden state is seamlessly steered: $\tilde{h}_t^{(l)} = h_t^{(l)} + \alpha_t^{(l)} \mathbf{v}_{r}^{(l)}$, where the intensity $\alpha_t^{(l)}$ dynamically decreases as visual reliance increases:
\begin{equation}
\label{eq:mars_steering}
\small
\alpha_t^{(l)} = \lambda \cdot \max \left(0, 1 - \frac{\omega_t^{(l)}}{\tau_{vis}} \right),
\end{equation}
where $\lambda$ is the base coefficient and $\tau_{vis}$ is a saturation margin. Appendix~\ref{app:thm1} formally proves that Eq.~\ref{eq:mars_steering} is the exact closed-form solution to resolve the trade-off between character drift and visual hallucination risks.

\subsection{Integration-Level: Dynamic Contextual Re-Injection}
\label{sec:reinjection}
To prevent visual forgetting in deep layers where textual priors overwhelm visual grounding, we bypass the standard self-attention bottleneck by directly injecting the purified orthogonal visual subset $\tilde{\mathbf{Z}} = [\tilde{z}_1, \dots, \tilde{z}_Q]$ into the intermediate Feed-Forward Networks (FFN) \citep{zou2025look}. The forward pass is residually augmented:
\begin{equation}
\label{eq:reinjection}
\small
\text{FFN}(h_t^{(l)}|\tilde{\mathbf{Z}}) = \text{FFN}(\tilde{h}_t^{(l)}) + \eta \cdot \phi \left( \frac{h_t^{(l)} \tilde{\mathbf{Z}}^\top}{\sqrt{d}} \right) \tilde{\mathbf{Z}},
\end{equation}
where $\phi(\cdot)$ is the native non-linear activation function acting as unnormalized similarity gating, and $\eta \in (0, 1]$ bounds the injection strength. Because OOC noise within $\tilde{\mathbf{Z}}$ has already been structurally eliminated by OFM, this mechanism securely enforces character-consistent factual grounding without reintroducing interference.

\section{Experiments}
\label{sec:experiments}

\subsection{Experimental Settings}\label{sec: experimenta_settings}

\paragraph{Benchmarks and Datasets.}
We evaluate on MMRole, a multimodal role-playing benchmark that pairs character instructions with image-grounded questions. MMRole scores each answer along three axes: general dialogue quality, visual-textual grounding, and role-playing consistency. Concretely, the official evaluator reports instruction adherence (IA), fluency (Flu), coherency (Coh), image-text relevance (ITR), response accuracy (RA), personality consistency (PC), knowledge consistency (KC), and tone consistency (TC). The final score is defined as the rating ratio between model-generated responses and reference responses. This design mitigates systematic scoring bias of the evaluator model. Detailed definitions are summarized in Table~\ref{tab:mmrole_metrics}. The dataset contains 85,456 training samples, 216 in-domain test samples, and 78 out-of-domain test samples, covering 85 total characters and three dialogue scenarios: image comments, human-role conversations, and inter-role conversations. 

\paragraph{Implementation.}
All experiments are implemented with PyTorch and HuggingFace Transformers. 
The hyperparameters and operation vectors of CAVI and other baselines are tuned on 100 randomly sampled training examples and fixed across all test runs. Full hyperparameter list is reported in Appendix~\ref{app:experimental_details}. The main experiments report results on Qwen3-VL-8B-Instruct and Gemma-3n-E4B-it; ablations and analysis are run on Qwen3-VL-8B-Instruct unless otherwise stated.

\paragraph{Baselines.}
We categorize the baselines into two main groups. (1) \textbf{Open-Weight Models and Visual Hallucination Mitigation Methods:} We first include supervised fine-tuning (\textbf{SFT}) on the MMRole training set. To compare against general hallucination mitigation techniques, we evaluate \textbf{MemVR}~\citep{zou2025look}, which uses memory-guided visual refinement for more faithful responses; \textbf{ClearSight}~\citep{yin2025clear}, which enhances visual grounding through evidence-aware generation; \textbf{AIR}~\citep{zhu2026look}, which adaptively revises image-conditioned responses for robustness; and \textbf{VCD}~\citep{leng2024mitigating}, which mitigates hallucination by contrasting visual contexts.  (2) \textbf{Closed-Source Large MLLMs:} We evaluate leading proprietary models, including \textbf{GPT-5}~\citep{singh2026openaigpt5card}, \textbf{Gemini-Pro-3}~\citep{deepmind2025gemini3}, and \textbf{Claude-Opus-4}~\citep{anthropic2025claude4} via their APIs, to establish performance upper bounds.

\subsection{Main Results}

\paragraph{In-domain Results.}
As shown in Table~\ref{tab:in_test_main_results}, \textbf{\M} achieves the highest overall scores among open-weight models, approaching massive proprietary models like GPT-5. Crucially, the SFT improves surface-level fluency but severely degrades role consistency. This confirms that standard multimodal role-playing datasets, which lack expensive, character-specific visual perspectives, fail to resolve and even exacerbate \textbf{Modality-Role Interference (MRI)}. Furthermore, traditional visual hallucination mitigations (e.g., VCD, MemVR) improve visual relevance but struggle to maintain persona traits due to their \textbf{character-agnostic} feature. In contrast, \textbf{\M} stands out by seamlessly synthesizing rigorous visual grounding with vivid role consistency without parameter updates.

\begin{table}[t]
\centering
\scriptsize
\caption{Main MMRole in-domain results. IA: instruction adherence; Flu: fluency; Coh: coherency; ITR: image-text relevance; RA: response accuracy; PC/KC/TC: consistency for personality/knowledge/tone. The best results within each model group are \textbf{bolded}. The final metric is the model-to-reference score ratio, where even marginal numerical gains signify substantial improvements.}
\label{tab:in_test_main_results}
\begin{tabular}{lccccccccc}
\toprule
Method & IA & Flu & Coh & ITR & RA & PC & KC & TC & Overall \\
\midrule
\multicolumn{10}{c}{\textit{Qwen3-VL-8B}} \\
\midrule
Base & 1.062 & 0.987 & 1.072 & 1.097 & 1.078 & 1.180 & 1.122 & 1.176 & 1.097 \\
SFT & \textbf{1.143} & \textbf{1.149} & 1.145 & 1.119 & 1.139 & 1.109 & 1.117 & 1.098 & 1.127 \\
MemVR & 1.090 & 1.001 & 1.099 & 1.133 & 1.107 & 1.218 & 1.145 & 1.213 & 1.126 \\
ClearSight & 1.097 & 1.011 & 1.102 & 1.132 & 1.099 & 1.219 & 1.146 & 1.220 & 1.128 \\
VCD & 1.109 & 1.031 & 1.120 & 1.154 & 1.135 & 1.229 & 1.171 & 1.233 & 1.148 \\
AIR & 1.087 & 1.016 & 1.102 & 1.132 & 1.096 & 1.204 & 1.132 & 1.199 & 1.121 \\
\rowcolor{gray!15} \textbf{\M (Ours)} & 1.122 & 1.062 & \textbf{1.152} & \textbf{1.208} & \textbf{1.149} & \textbf{1.248} & \textbf{1.204} & \textbf{1.241} & \textbf{1.173} \\
\midrule
\multicolumn{10}{c}{\textit{Gemma-3n-E4B}} \\
\midrule
Base & 1.043 & 0.990 & 1.074 & 1.086 & 1.059 & 1.137 & 1.079 & 1.138 & 1.076 \\
SFT & 1.112  & \textbf{1.116} & 1.120 & 1.104 & 1.109 & 1.123 & 1.108  & 1.104 & 1.112 \\
MemVR & 1.073 & 1.018 & 1.104 & 1.121 & 1.090 & 1.169 & 1.111 & 1.169 & 1.107 \\
ClearSight & 1.098 & 1.020 & 1.110 & 1.115 & 1.108 & 1.176 & 1.120 & 1.173 & 1.115 \\
VCD & 1.104 & 1.048 & 1.129 & 1.137 & 1.122 & 1.204 & 1.137 & 1.191 & 1.134 \\
AIR & 1.073 & 1.018 & 1.104 & 1.121 & 1.090 & 1.169 & 1.111 & 1.169 & 1.107 \\
\rowcolor{gray!15} \textbf{\M (Ours)} & \textbf{1.135} & 1.066 & \textbf{1.156} & \textbf{1.195} & \textbf{1.151} & \textbf{1.217} & \textbf{1.179} & \textbf{1.209} & \textbf{1.164} \\
\midrule
\multicolumn{10}{c}{\textit{Closed large MLLMs}} \\
\midrule
GPT-5 & 1.120 & 1.094 & 1.142 & 1.157 & 1.163 & 1.232 & 1.172 & 1.227 & 1.163 \\
Gemini-Pro-3 & 1.097 & 1.103 & 1.120 & 1.105 & 1.119 & 1.163 & 1.116 & 1.162 & 1.123 \\
Claude-Opus-4 & {1.199} & {1.142} & {1.216} & {1.238} & {1.223} & {1.296} & {1.249} & {1.288} & {1.231} \\
\bottomrule
\end{tabular}
\end{table}

\paragraph{Out-of-domain Results.}
Table~\ref{tab:out_test_main_results} illustrates \textbf{\M}'s strong generalization to unseen characters. On out-of-domain data, existing fine-tuned models and training-free baselines struggle to achieve high role consistency, indicating that their static visual corrections or dataset-specific priors are insufficient for entirely new personas. Conversely, \textbf{\M} maintains its robust superiority across overall coherency and consistency metrics. This sustained top-tier performance proves that our dynamic cognitive filter effectively transfers the generalized ``character lens'' to novel roles in a zero-shot manner, consistently outperforming other open-weight alternatives.

\begin{table}[t]
\centering
\scriptsize
\caption{Main MMRole out-of-domain results.}
\label{tab:out_test_main_results}
\begin{tabular}{lccccccccc}
\toprule
Method & IA & Flu & Coh & ITR & RA & PC & KC & TC & Overall \\
\midrule
\multicolumn{10}{c}{\textit{Qwen3-VL-8B}} \\
\midrule
Base & 1.039 & 0.960 & 1.042 & 1.078 & 1.048 & 1.155 & 1.097 & 1.144 & 1.070 \\
SFT & \textbf{1.136} & \textbf{1.146} & 1.141 & 1.123 & 1.125 & 1.078 & 1.109 & 1.050 & 1.114 \\
MemVR & 1.066 & 0.992 & 1.071 & 1.118 & 1.094 & 1.213 & 1.138 & 1.197 & 1.111 \\
ClearSight & 1.056 & 0.998 & 1.073 & 1.113 & 1.078 & 1.200 & 1.134 & 1.204 & 1.107 \\
VCD & 1.095 & 1.027 & 1.120 & 1.135 & 1.113 & 1.206 & 1.152 & 1.206 & 1.132 \\
AIR & 1.085 & 1.004 & 1.100 & 1.140 & 1.114 & 1.204 & 1.146 & 1.195 & 1.124 \\
\rowcolor{gray!15} \textbf{\M (Ours)} & 1.111 & 1.047 & \textbf{1.143} & \textbf{1.178} & \textbf{1.126} & \textbf{1.227} & \textbf{1.177} & \textbf{1.221} & \textbf{1.154} \\
\midrule
\multicolumn{10}{c}{\textit{Gemma-3n-E4B}} \\
\midrule
Base & 1.020 & 0.996 & 1.063 & 1.070 & 1.050 & 1.118 & 1.063 & 1.109 & 1.061 \\
SFT & \textbf{1.120} & \textbf{1.121} & 1.118 & 1.110 & 1.111 & 1.088 & 1.104 & 1.071 & 1.105 \\
MemVR & 1.073 & 1.050 & 1.116 & 1.124 & 1.107 & 1.170 & 1.125 & 1.174 & 1.117 \\
ClearSight & 1.075 & 1.018 & 1.108 & 1.117 & 1.104 & 1.177 & 1.120 & 1.167 & 1.111 \\
VCD & 1.095 & 1.034 & 1.126 & 1.129 & 1.110 & 1.179 & 1.127 & \textbf{1.197} & 1.125 \\
AIR & 1.073 & 1.050 & 1.116 & 1.124 & 1.107 & 1.170 & 1.125 & 1.174 & 1.117 \\
\rowcolor{gray!15} \textbf{\M (Ours)} & 1.086 & 1.047 & \textbf{1.143} & \textbf{1.157} & \textbf{1.132} & \textbf{1.200} & \textbf{1.152} & 1.191 & \textbf{1.139} \\
\midrule
\multicolumn{10}{c}{\textit{Closed large MLLMs}} \\
\midrule
GPT-5 & 1.092 & 1.093 & 1.136 & 1.146 & 1.127 & 1.205 & 1.139 & 1.201 & 1.142 \\
Gemini-Pro-3 & 1.076 & 1.103 & 1.116 & 1.089 & 1.082 & 1.087 & 1.084 & 1.093 & 1.091 \\
Claude-Opus-4 & {1.164} & {1.120} & {1.195} & {1.201} & {1.185} & {1.239} & {1.189} & {1.246} & {1.192} \\
\bottomrule
\end{tabular}
\end{table}

\subsection{Ablation Studies}

\begin{wraptable}{R}{0.45\textwidth} 
    \centering
    \footnotesize
    \caption{Module ablations on Qwen3-VL-8B-Instruct.}
    \label{tab:qwen8b_ablation}
    \resizebox{\linewidth}{!}{
    \begin{tabular}{l r@{\hspace{0.5ex}}l r@{\hspace{0.5ex}}l}
        \toprule
        Variant & \multicolumn{2}{c}{In-domain} & \multicolumn{2}{c}{Out-of-domain} \\
        \midrule
        CAVI & 1.1730 & & 1.1537 & \\
        \quad\textit{w/o} CTP & 1.1623 & (\textcolor{red}{$\downarrow$1.1\%}) & 1.1455 & (\textcolor{red}{$\downarrow$0.8\%}) \\
        \quad\textit{w/o} OFM & 1.1585 & (\textcolor{red}{$\downarrow$1.5\%}) & 1.1357 & (\textcolor{red}{$\downarrow$1.8\%}) \\
        \quad\textit{w/o} MARS & 1.1659 & (\textcolor{red}{$\downarrow$0.7\%}) & 1.1320 & (\textcolor{red}{$\downarrow$2.2\%}) \\
        \quad\textit{w/o} Re-Inject & 1.1639 & (\textcolor{red}{$\downarrow$0.9\%}) & 1.1381 & (\textcolor{red}{$\downarrow$1.6\%}) \\
        \bottomrule
    \end{tabular}}
\end{wraptable}

We ablate the three core modules by disabling CTP, OFM, and MARS independently. Table~\ref{tab:qwen8b_ablation} shows that removing any component hurts performance on both splits, confirming that the gains are not due to a single isolated trick. OFM produces the largest in-domain drop, suggesting that suppressing orthogonal visual noise is important even when character identities are familiar. On the out-of-domain split, removing MARS causes the largest degradation, followed by OFM and visual re-injection, indicating that adaptive role steering and purified visual memory become more important when the character-image pairing shifts. CTP has a smaller but consistent effect, which is expected because it mainly controls the quality of the visual evidence passed to later modules rather than directly changing the decoding state.

\begin{figure}[t]
    \centering
    \begin{minipage}{0.52\textwidth}
        \centering
        \includegraphics[width=\linewidth]{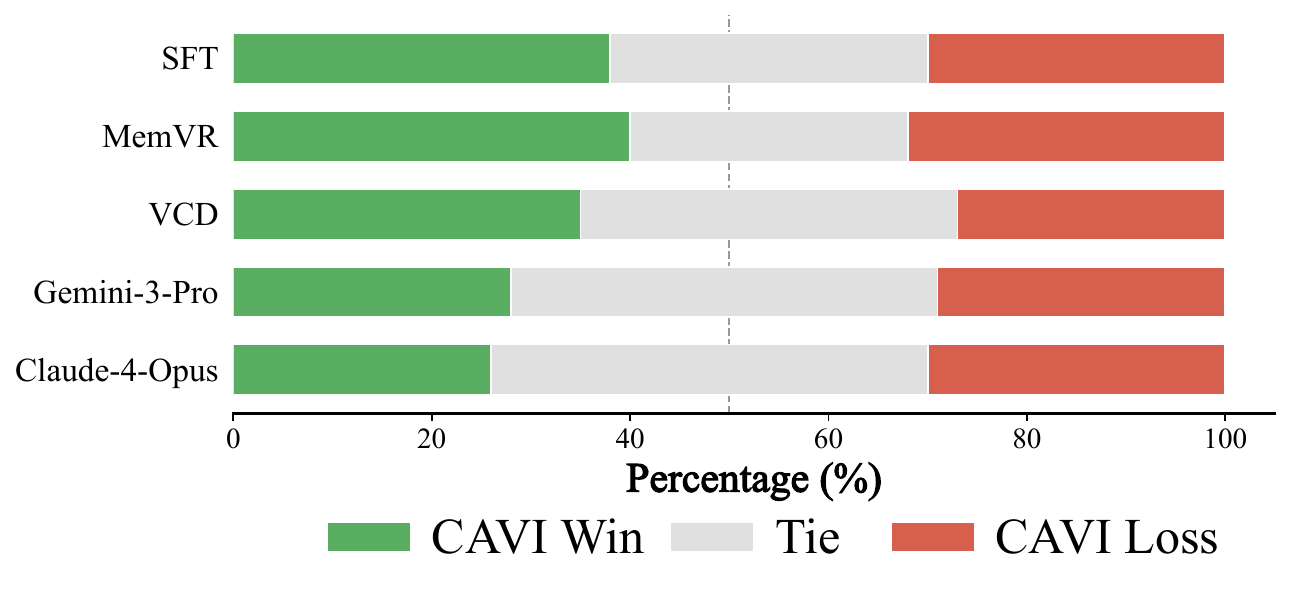}
        \caption{Human evaluation results.}
        \label{fig: human_eval}
    \end{minipage}\hfill 
    \begin{minipage}{0.42\textwidth}
        \centering
        \includegraphics[width=\linewidth]{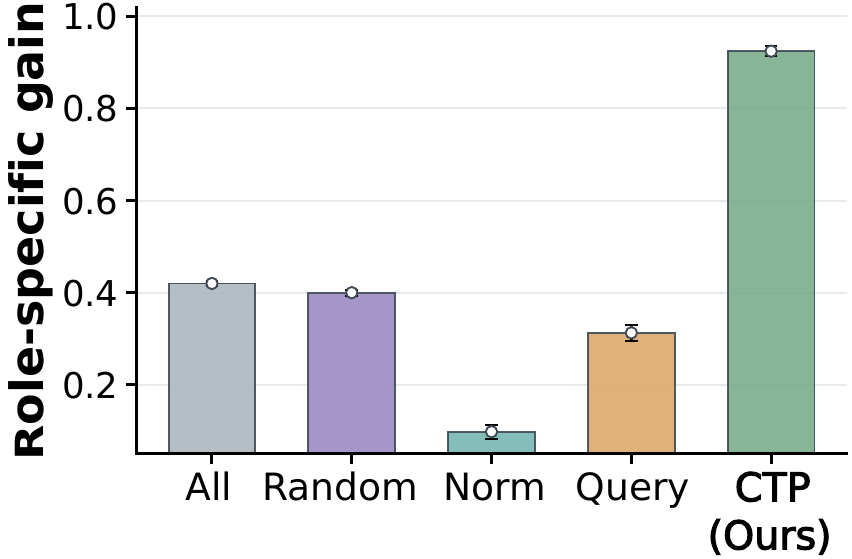}
        \caption{Role-specific visual gain of token selectors. The displayed score is per-example normalized.}
        \label{fig:ctp_resonance}
    \end{minipage}
\end{figure}

\subsection{Human Evaluation}
To validate real-world utility, we conduct a blinded, pairwise human evaluation involving 4 experts across 100 interaction sessions. As shown in Figure~\ref{fig: human_eval}, \textbf{\M} achieves leading performance, outperforming other methods under a strict criterion requiring superiority in both \textit{Character Visual Grounding} and \textit{Role Consistency}. Detailed analysis and evaluation protocols are provided in Appendix~\ref{app: human_eval}.

\subsection{Further Analysis}
\label{sec:analysis}



\paragraph{Character-Guided Token Pruning.}
CTP should retain visual evidence that is compatible with the current character identity rather than preserving a generic visual field. We therefore compute role-specific resonance, namely the cosine similarity between selected visual tokens and the role anchor minus their similarity to the user-query anchor. We compare CTP against using all tokens, a random subset with the same cardinality, a norm-based salience selector, and a query-relevance selector. As illustrated in Figure~\ref{fig:ctp_resonance}, CTP obtains the strongest role-specific gain, indicating that it does not merely prune visual tokens but preferentially keeps character-relevant evidence.


\begin{figure*}[t]
    \centering
    \begin{minipage}{0.29\textwidth}
        \centering
        \includegraphics[width=\linewidth]{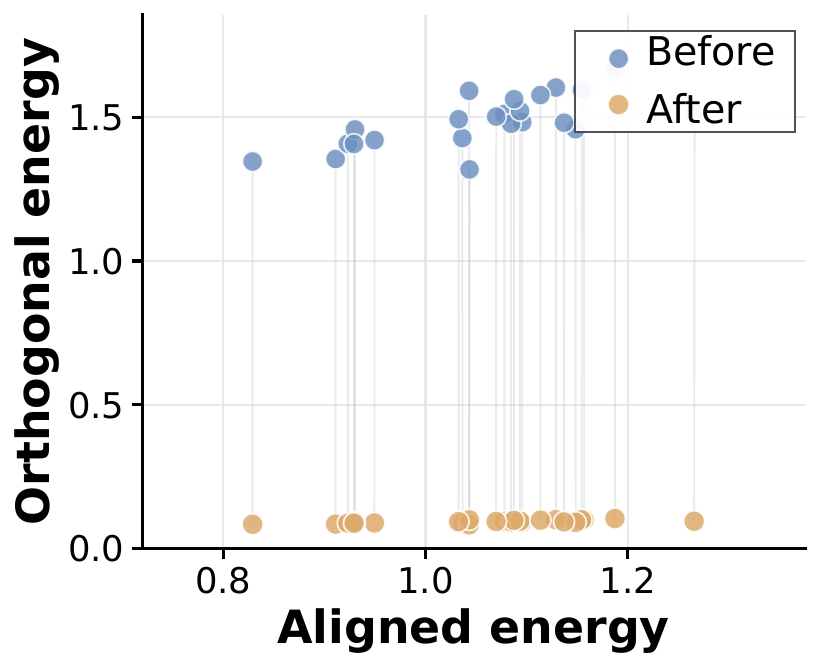}
    \end{minipage}
    \hfill
    \begin{minipage}{0.32\textwidth}
        \centering
        \includegraphics[width=\linewidth]{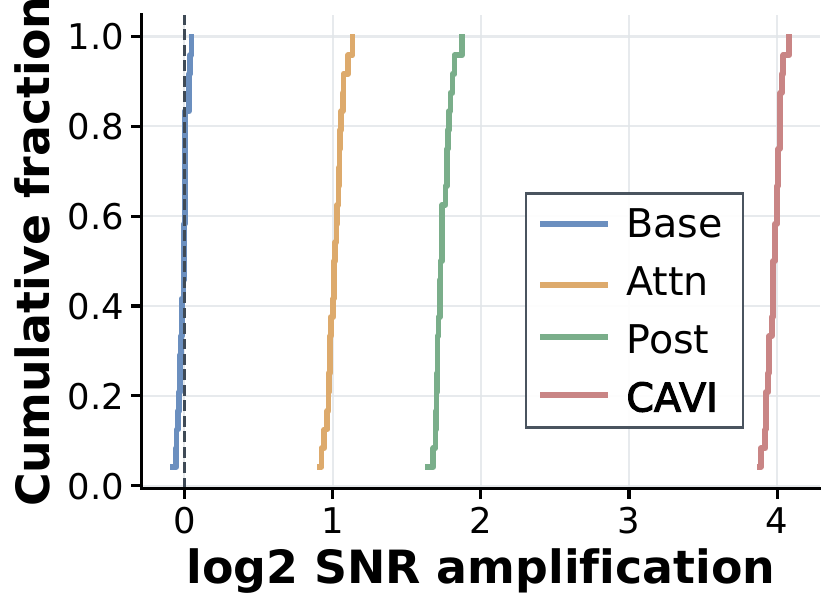}
    \end{minipage}
    \hfill
    \begin{minipage}{0.32\textwidth}
        \centering
        \includegraphics[width=\linewidth]{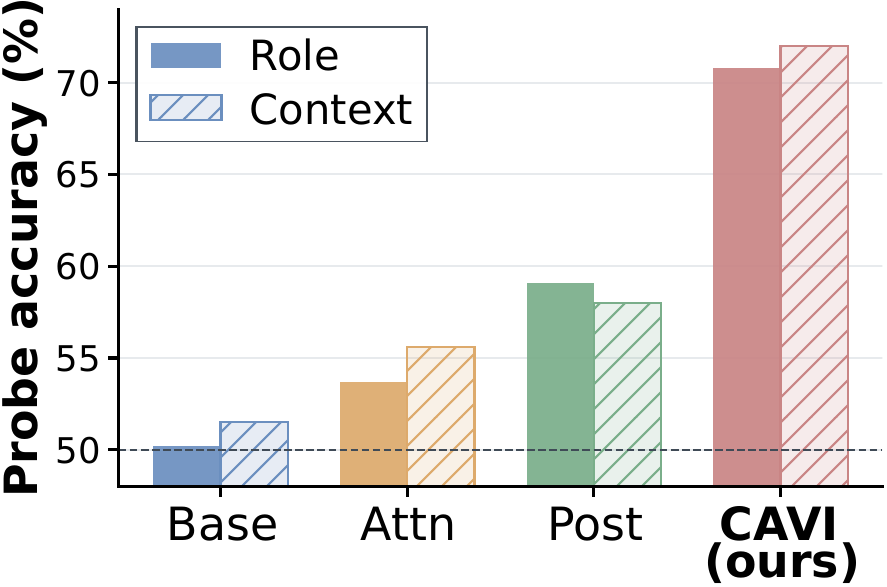}
    \end{minipage}
    \caption{OFM diagnostics. Left: aligned-versus-orthogonal energy before and after modulation. Middle: empirical distribution of log-scale SNR amplification. Right: probe separability for role and context categories.}
    \label{fig:ofm_analysis}
\end{figure*}

\begin{figure}[t]
    \centering
    \subfloat[Base \label{fig:base_tsne}]{
        \includegraphics[width=0.32\linewidth]{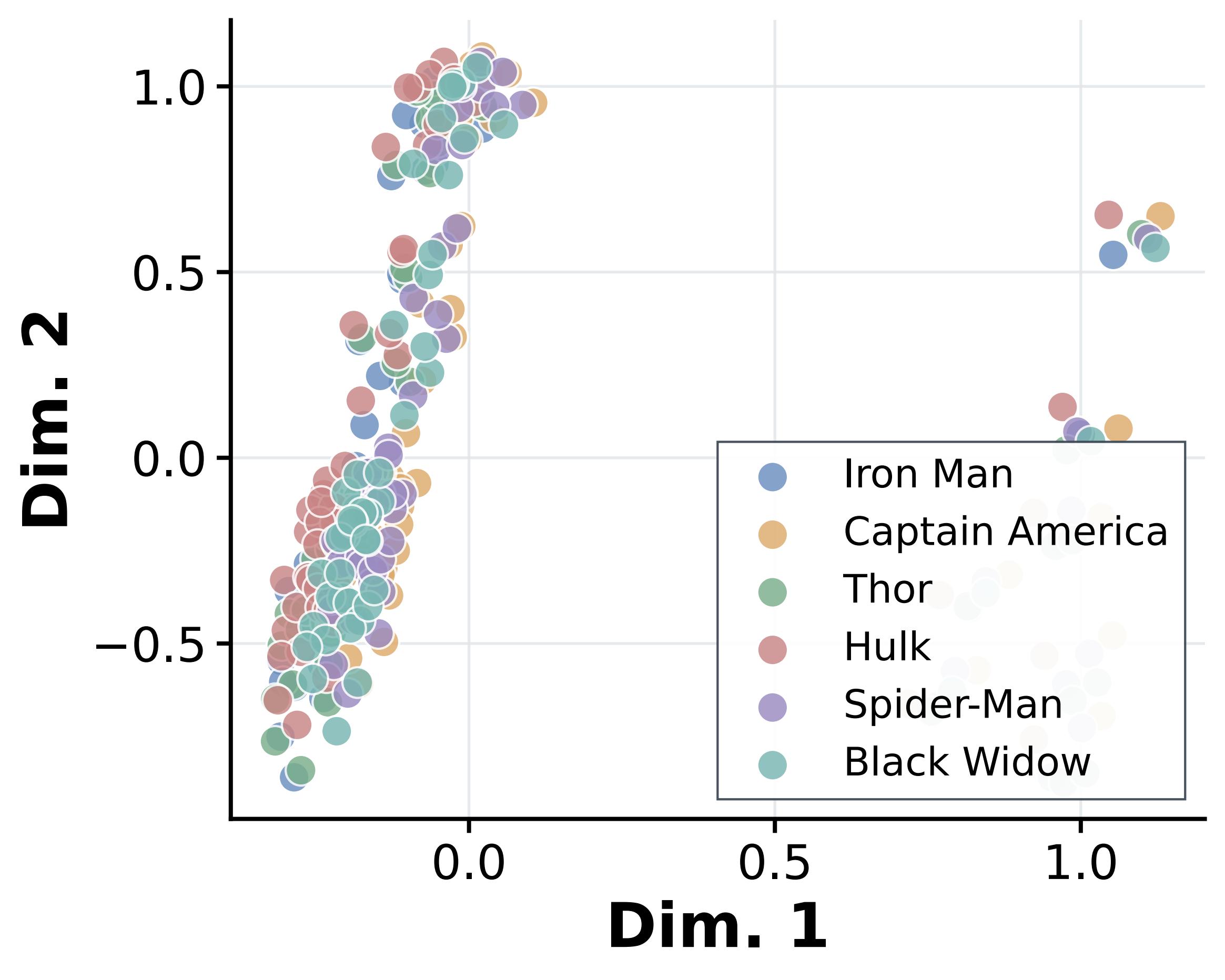}
    }
    \hfill
    \subfloat[VCD \label{fig:vcd_tsne}]{
        \includegraphics[width=0.32\linewidth]{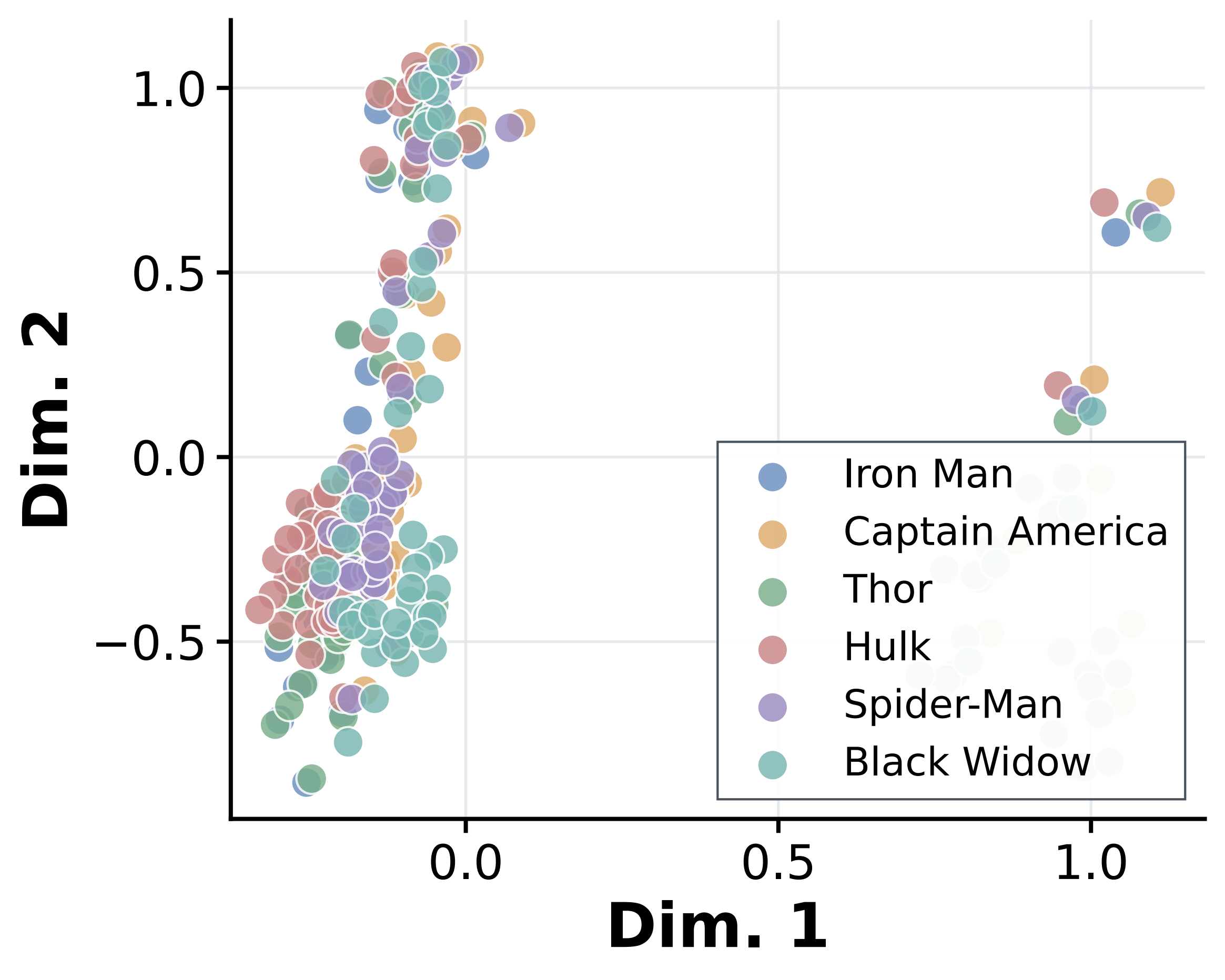}
    }
    \hfill
    \subfloat[\textbf{\M (Ours)} \label{fig:cavi_tsne}]{
        \includegraphics[width=0.32\linewidth]{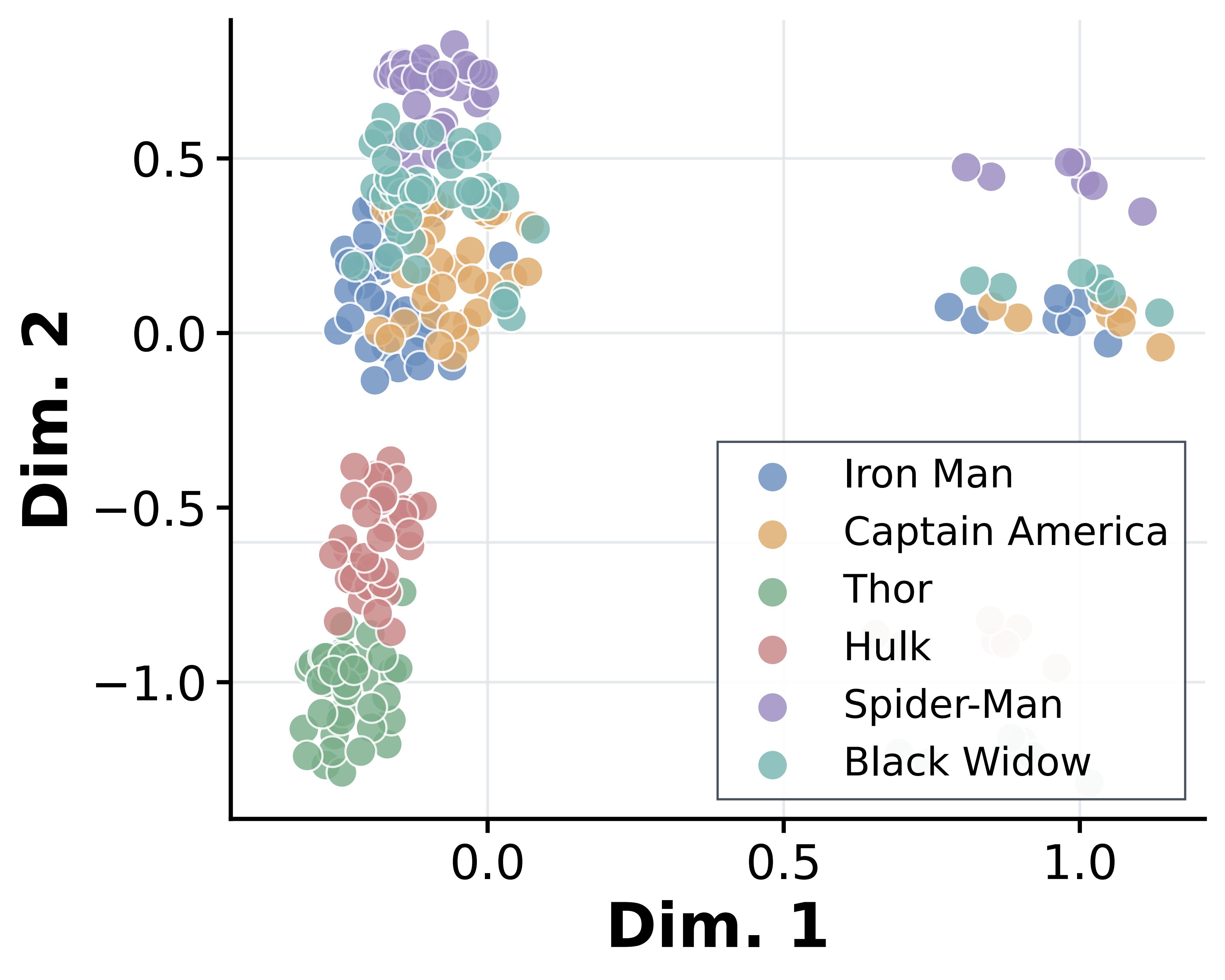}
    }
    \caption{Controlled role-swap t-SNE for hidden states of Base, VCD, and CAVI across 6 roles.}
    \label{fig:controlled_role_tsne}
\end{figure}

\begin{figure*}[t]
    \centering
    \begin{minipage}{0.195\textwidth}
        \centering
        \includegraphics[width=\linewidth]{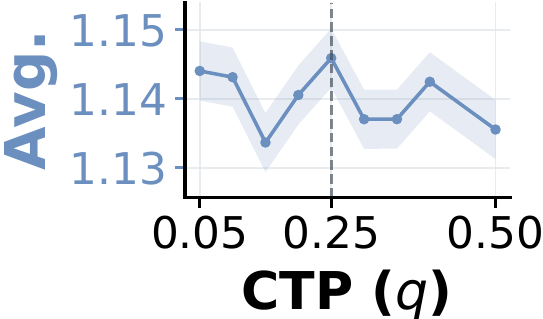}
    \end{minipage}
    \begin{minipage}{0.195\textwidth}
        \centering
        \includegraphics[width=\linewidth]{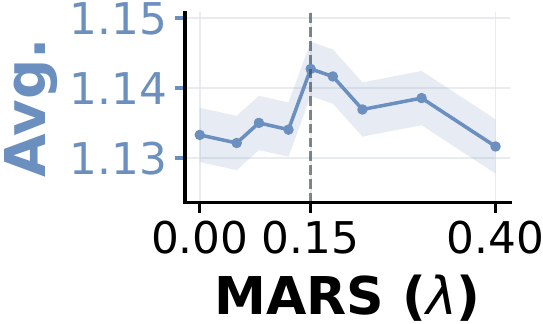}
    \end{minipage}
    \begin{minipage}{0.195\textwidth}
        \centering
        \includegraphics[width=\linewidth]{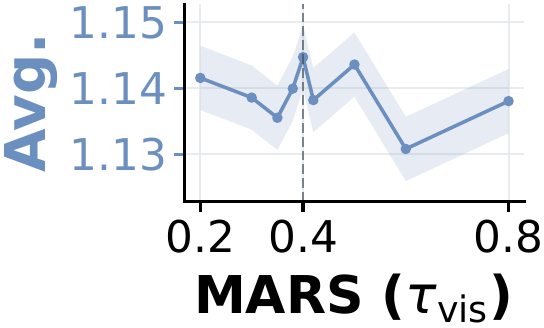}
    \end{minipage}
    \begin{minipage}{0.195\textwidth}
        \centering
        \includegraphics[width=\linewidth]{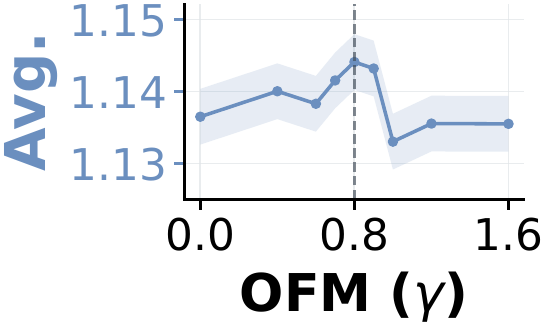}
    \end{minipage}
    \begin{minipage}{0.195\textwidth}
        \centering
        \includegraphics[width=\linewidth]{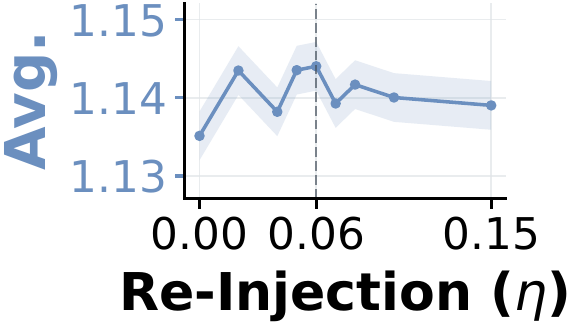}
    \end{minipage}
    \caption{Hyperparameter sensitivity evaluated on 100 sampled in-domain test examples. From left to right: CTP budget $q$, MARS base coefficient $\lambda$, MARS saturation margin $\tau_{\mathrm{vis}}$, OFM noise leakage $\gamma$, and visual re-injection strength $\eta$. Dashed lines mark the default configurations independently tuned on 100 training samples in the main result.}
    \label{fig:hyperpara}
\end{figure*}

\paragraph{Orthogonal Feature Modulation.}

OFM is designed to suppress the out-of-character component inside retained visual token features while preserving the component aligned with the character-context subspace. Figure~\ref{fig:ofm_analysis} evaluates this from three angles. The left panel compares aligned and orthogonal energy before and after OFM, showing that the gain is not a scalar rescaling artifact, as OFM selectively reduces orthogonal energy while retaining the aligned counterpart. The middle panel plots the empirical distribution of log-scale Signal-to-Noise Ratio (SNR) amplification, where the rightward shift of the curves indicates that the aligned component becomes increasingly dominant. The right panel reports probe accuracy for role and context categories. While the original representation (Base) stays near chance, features become progressively more separable following visual token selection (Attn), OFM (Post), and the full CAVI pipeline.

\paragraph{Modality-Adaptive Role Steering.}
MARS should make character identity more accessible during decoding. We build a controlled role-swap set by crossing 50 MMRole image-question pairs with six role prompts, yielding 300 visual-conditioned hidden-state points per method. As illustrated in Figure~\ref{fig:controlled_role_tsne}, the baselines show weak role structure but are still largely organized by shared visual-question geometry, while CAVI forms clearer role-conditioned neighborhoods without removing input-specific variation. This supports the intended behavior of MARS, role identity becomes more explicit in the decoding state.



\begin{table}[t] 
    \centering
    \footnotesize

    \begin{minipage}[t]{0.46\linewidth} 
        \centering
        \renewcommand{\arraystretch}{1.4} 
        \caption{Average metrics for different layer ranges.}
        \label{tab:layer_stats}
        \begin{tabular}{lc | @{\hspace{4ex}} lc} 
            \toprule
            $\mathcal{L}_{steer}$ & Overall $\uparrow$  & $\mathcal{L}_{steer}$ & Overall $\uparrow$\\
            \midrule
            16-32 & 1.1400 & 24-32          & 1.1356  \\
            18-32 & 1.1417 & \textbf{26-32} & \textbf{1.1423} \\
            20-32 & 1.1289 & 28-32          & 1.1321 \\
            22-32 & 1.1343 & 30-32          & 1.1297 \\
            \bottomrule
        \end{tabular}
    \end{minipage}
    \hfill 
    \begin{minipage}[t]{0.52\linewidth}
        \centering
        \renewcommand{\arraystretch}{1.1}
        \caption{Inference efficiency for Qwen3-VL-8B-Instruct.}
        \label{tab:runtime_perf}
        \resizebox{\linewidth}{!}{
        \begin{tabular}{lrrrr}
            \toprule
            Method & Tok./s $\uparrow$ & Sec./sample $\downarrow$ & Peak CUDA GB $\downarrow$ \\
            \midrule
            Base       & 13.77 & 27.28 & 19.22 \\
            VCD        & 6.03  & 59.11 & 24.87 \\
            MemVR      & 11.44 & 31.43 & 20.91 \\
            ClearSight & 11.19 & 32.52 & 21.30 \\
            AIR        & 12.46 & 30.70 & 20.08 \\
            \textbf{CAVI (Ours)}       & 12.74 & 28.36 & 20.79 \\
            \bottomrule
        \end{tabular}}
    \end{minipage}
\end{table}


\paragraph{Hyperparameter Sensitivity.}
As shown in Figure~\ref{fig:hyperpara}, the hyperparameters generally exhibit an inverted-U trend, reflecting the underlying information bottleneck trade-off. Overly aggressive interventions (e.g., an excessively low CTP budget $q$ or high MARS coefficient $\lambda$) discard critical visual facts or indiscriminately overwrite visual grounding, risking hallucinations. Conversely, conservative settings fail to adequately filter Out-Of-Character noise or correct role drift. Furthermore, Table~\ref{tab:layer_stats} demonstrates that confining the steering and re-injection intervention ($\mathcal{L}_{steer}$) to deeper layers (specifically 26--32) yields the optimal overall performance. This aligns precisely with the representational dynamics of MLLMs: early layers are best preserved for low-level cross-modal alignment and foundational visual perception, whereas deeper layers govern high-level semantic synthesis and persona realization. Therefore, mid-to-late layers serve as the optimal locus for explicit cognitive role steering without disrupting fundamental visual integrity.

\paragraph{Efficiency Analysis.}
Table~\ref{tab:runtime_perf} compares the inference efficiency of various methods on Qwen3-VL-8B-Instruct using a single GPU. Since \textbf{\M} relies purely on lightweight closed-form subspace projections and dynamic activation steering, it introduces minimal computational overhead. Compared to the original Base model, \textbf{\M} maintains a highly competitive generation throughput and incurs a marginal peak memory increase. In stark contrast, contrastive decoding interventions like VCD suffer a massive 56\% throughput degradation and high memory inflation due to the heavy burden of dual-branch forward passes. This explicitly indicates that \textbf{\M} provides substantial role-alignment gains with negligible latency, remaining highly practical and scalable for real-time deployment.

\subsection{Case Study}

To qualitatively illustrate the capabilities of \M, we present dialogue examples in Table~\ref{tab:case_howard_child} and Table~\ref{tab:case_howard_clocks}. In the first example, Howard Wolowitz is asked to comment on a child eating in warm light. Base, VCD, and MemVR mostly describe a pleasant photograph or snack scene, while CAVI grounds the same evidence in Howard's aerospace-engineering and fatherhood persona, e.g., turning the snack moment into an edible experiment and a NASA-like thermal image. In the second example, Sheldon praises antique clocks as precise mechanical systems. Generic visual baselines either dismiss the clocks as decorative objects or fall back to surface-level banter, whereas CAVI first acknowledges concrete visual cues--the octagonal face, wood grain, Roman numerals, cuckoo clock, and pendulum--and then responds through Howard's engineering worldview. These cases illustrate that CAVI does not merely intensify role style; it uses role-conditioned visual evidence to keep answers both image-grounded and persona-consistent.

\section{Conclusion}
In this paper, we formalize Modality-Role Interference (MRI) as the fundamental cognitive mismatch between objective visual perception and subjective character personas in MLLMs. To resolve this, we introduce Character-Aware Visual Intervention (\textbf{CAVI}), a training-free framework that enables agents to perceive environments \textit{``Through the Lens of Character.''} By functioning as an information bottleneck, \textbf{CAVI} systematically purifies role-consistent visual facts from massive visual inputs. Extensive experiments validate that \textbf{CAVI} effectively mitigates Modality-Role Interference, seamlessly harmonizing what the model objectively sees with who the character subjectively is. Ultimately, our findings suggest that deterministic, inference-time feature modulation offers a highly viable and lightweight pathway for complex multimodal persona alignment.



\medskip

{
\small

\bibliographystyle{unsrtnat}
\bibliography{ref}
}


\appendix

\appendix

\section{Theoretical Analysis and Proofs}
\label{app:theoretical_proofs}

This appendix provides the formal information-theoretic context, propositions, geometric assumptions, and detailed proofs supporting the methodology presented in Section~\ref{sec:method}.


\subsection{Ineffectiveness of Scalar Attention for Feature Disentanglement}
\label{app:prop1}

\begin{proposition}[Invariant Noise Leakage in Scalar Attention Weighting]
\label{prop:noise_leakage}
Let $\hat{z} \in \mathbb{R}^d$ be a visual token decomposed into an orthogonal direct sum $\hat{z} = z^{role} + z^{noise}$, where $\langle z^{role}, z^{noise} \rangle = 0$ and $\|z^{noise}\|_2 > 0$. We define the intra-token Signal-to-Noise Ratio (SNR) as $\text{SNR}(\hat{z}) = \frac{\|z^{role}\|_2^2}{\|z^{noise}\|_2^2}$. 
Under standard attention mechanisms, the token is scaled by a weight $\alpha \in (0,1]$, yielding the context vector contribution $\tilde{z}_{attn} = \alpha z^{role} + \alpha z^{noise}$. The resulting SNR remains strictly invariant: $\text{SNR}(\tilde{z}_{attn}) = \text{SNR}(\hat{z})$.
\end{proposition}

\textit{Proof of Proposition \ref{prop:noise_leakage}.} 
For scalar attention, the modulated SNR is computed as:
\begin{equation}
\text{SNR}(\tilde{z}_{attn}) = \frac{\|\alpha z^{role}\|_2^2}{\|\alpha z^{noise}\|_2^2} = \frac{\alpha^2 \|z^{role}\|_2^2}{\alpha^2 \|z^{noise}\|_2^2} = \frac{\|z^{role}\|_2^2}{\|z^{noise}\|_2^2} = \text{SNR}(\hat{z})
\end{equation}
To maintain necessary visual grounding and prevent object hallucinations, the model must assign an attention weight $\alpha > 0$. Consequently, the absolute noise leakage $\|\alpha z^{noise}\|_2 > 0$ strictly persists and is proportionally bounded by the factual signal. Therefore, scalar weighting is mathematically incapable of decoupling features at the sub-token level. $\hfill \blacksquare$

\subsection{Sufficiency of Orthogonal Feature Modulation (OFM)}
\label{app:prop2}

To mathematically guarantee the effectiveness of OFM, we introduce the following geometric assumptions common in high-dimensional latent space analysis.

\newtheorem{assumption}{Assumption}
\begin{assumption}[High-Dimensional Semantic Geometry]
\label{assum:geometry}
In the $d$-dimensional latent space of MLLMs ($d \gg 1$), we assume:
\begin{itemize}
    \item \textbf{Signal Alignment:} The character-consistent visual fact $z^{role}$ is semantically driven by $\mathbf{e}_c$ and $h_t^{(l)}$. Thus, it is substantially spanned by $\mathcal{V}_{CC}$, retaining most of its energy after projection: $\|\mathbf{P}_{CC} z^{role}\|_2^2 \ge (1-\delta)\|z^{role}\|_2^2$ for a small misalignment margin $\delta \in (0,1)$.
    \item \textbf{Noise Isotropy:} The OOC noise $z^{noise}$ represents context-independent generic visual attributes. It behaves as an isotropic random vector with respect to the specific 2D character-context subspace $\mathcal{V}_{CC}$.
\end{itemize}
\end{assumption}

\begin{proposition}[Sufficiency of Projection via Dimensionality Blessing]
\label{prop:sufficiency}
Under Assumption \ref{assum:geometry}, let $\gamma \to 0$. Defining the expected SNR as the ratio of signal energy to the expected noise energy, the orthogonal projection onto the rank-2 subspace $\mathcal{V}_{CC}$ mathematically amplifies the intra-token SNR by a factor proportional to $\frac{d}{2}$:
\begin{equation}
\mathbb{E}[\text{SNR}_{OFM}] \ge \frac{d}{2} (1-\delta) \cdot \text{SNR}(\hat{z}) \gg \text{SNR}(\hat{z})
\end{equation}
\end{proposition}

\textit{Proof of Proposition \ref{prop:sufficiency}.} 
The retained feature component after orthogonal projection is $\tilde{z} = \mathbf{P}_{CC} z^{role} + \mathbf{P}_{CC} z^{noise}$. Based on Assumption \ref{assum:geometry}, since $z^{noise}$ is isotropic in $\mathbb{R}^d$, its expected covariance matrix is $\Sigma = \frac{\mathbb{E}[\|z^{noise}\|_2^2]}{d} \mathbf{I}$.
When projecting this isotropic noise onto the rank-2 character-context subspace $\mathcal{V}_{CC}$, the expected squared norm of the noise residual can be determined using the properties of the trace operator:
\begin{equation}
\mathbb{E}[\|\mathbf{P}_{CC} z^{noise}\|_2^2] = \text{Tr}\big(\mathbf{P}_{CC} \Sigma\big) = \frac{\mathbb{E}[\|z^{noise}\|_2^2]}{d} \text{Tr}(\mathbf{P}_{CC})
\end{equation}
Since $\mathbf{P}_{CC}$ is an orthogonal projection matrix onto a 2-dimensional subspace, its trace equals its rank: $\text{Tr}(\mathbf{P}_{CC}) = 2$. Thus:
\begin{equation}
\mathbb{E}[\|\mathbf{P}_{CC} z^{noise}\|_2^2] = \frac{2}{d} \mathbb{E}[\|z^{noise}\|_2^2]
\end{equation}
Substituting the bounded signal energy (which retains at least $(1-\delta)$ of its original energy) and the heavily dimensionally-attenuated noise energy into the expected SNR formulation yields:
\begin{equation}
\mathbb{E}[\text{SNR}_{OFM}] = \frac{\|\mathbf{P}_{CC} z^{role}\|_2^2}{\mathbb{E}[\|\mathbf{P}_{CC} z^{noise}\|_2^2]} \ge \frac{(1-\delta)\|z^{role}\|_2^2}{\frac{2}{d} \mathbb{E}[\|z^{noise}\|_2^2]} = \frac{d}{2} (1-\delta) \text{SNR}(\hat{z})
\end{equation}
$\hfill \blacksquare$

\textbf{Quantitative Effectiveness:} In modern MLLMs (e.g., LLaMA-based architectures), the hidden dimension is typically $d \ge 4096$. Consequently, the geometric amplification factor is $\frac{d}{2} \ge 2048 \gg 1$. Even if the factual signal is heavily misaligned ($\delta = 0.5$, meaning half the signal is lost), projecting it onto $\mathcal{V}_{CC}$ still physically crushes the noise variance by orders of magnitude, effectively amplifying the expected SNR by over $1000\times$.

\subsection{Theoretical Justification of Modality-Adaptive Role Steering (MARS)}
\label{app:thm1}

\begin{proposition}[Risk Minimization Formulation]
\label{prop:mars_optimality}
During auto-regressive decoding, the agent faces a trade-off at layer $l$: minimizing the Character Drift Risk $\mathcal{R}_{drift}(\alpha) = \frac{1}{2}(\lambda - \alpha)^2$ versus minimizing the Visual Hallucination Risk $\mathcal{R}_{hal}(\alpha) = \mu \alpha \omega_t^{(l)}$, where $\alpha \ge 0$ is the steering intensity and $\mu > 0$ is a trade-off penalty.
The proposed MARS steering strategy naturally emerges as the closed-form solution that minimizes the joint risk functional $\mathcal{J}(\alpha) = \mathcal{R}_{drift}(\alpha) + \mathcal{R}_{hal}(\alpha)$.
\end{proposition}

\textit{Derivation of Proposition \ref{prop:mars_optimality}.}
We aim to find the appropriate intensity $\alpha^*$ that minimizes the joint risk:
\begin{equation}
\alpha^* = \mathop{\arg\min}_{\alpha \ge 0} \left[ \frac{1}{2}(\lambda - \alpha)^2 + \mu \alpha \omega_t^{(l)} \right]
\end{equation}
Taking the first derivative of the unconstrained objective $\mathcal{J}(\alpha)$ with respect to $\alpha$ and setting it to zero yields:
\begin{equation}
\frac{\partial \mathcal{J}}{\partial \alpha} = -(\lambda - \alpha) + \mu \omega_t^{(l)} = 0 \quad \implies \quad \alpha = \lambda - \mu \omega_t^{(l)}
\end{equation}
By defining the saturation margin threshold as $\tau_{vis} = \frac{\lambda}{\mu}$ and substituting this back into the equation, we obtain:
\begin{equation}
\alpha = \lambda - \frac{\lambda}{\tau_{vis}} \omega_t^{(l)} = \lambda \left( 1 - \frac{\omega_t^{(l)}}{\tau_{vis}} \right)
\end{equation}
To satisfy the non-negative physical constraint $\alpha \ge 0$ (which prevents the model from inversely applying the role steering vector, thereby avoiding negative intervention effects), we project the unconstrained solution onto the feasible domain. This directly leads to the exact formulation of MARS presented in Eq.~\ref{eq:mars_steering}:
\begin{equation}
\alpha^* = \lambda \cdot \max \left(0, 1 - \frac{\omega_t^{(l)}}{\tau_{vis}} \right)
\end{equation}
This derivation demonstrates that the proposed MARS is not a mere empirical heuristic, but a closed-form analytical solution derived to balance the decoding trade-off. $\hfill \blacksquare$

\section{Limitations}
\label{app: limitations}
While our proposed \textbf{CAVI} method demonstrates highly promising results as an efficient intervention for Modality-Role Interference (MRI), several aspects present exciting opportunities for future exploration:

First, as a fully training-free and plug-and-play approach, the upper bound of \textbf{CAVI}'s performance naturally correlates with the inherent vision-language alignment and instruction-following capabilities of the underlying base MLLM. Consequently, future advancements in foundation models will seamlessly unlock even stronger out-of-the-box multimodal role-playing performance without requiring any structural changes to our method.

Second, the Character Anchor is currently constructed by mean-pooling the textual profile's hidden states. This design was deliberately chosen for its computational efficiency and robust stability, which has proven highly effective across standard role-playing benchmarks. Future work could explore multi-vector or hierarchical representations to capture the subtle semantic nuances of exceptionally long character bibles (e.g., full-novel contexts) while carefully balancing inference overhead.

Third, the Modality-Adaptive Role Steering (MARS) module utilizes a lightweight offline calibration step to extract the expected perturbation direction. This principled design explicitly guarantees zero added latency during auto-regressive decoding, making it highly practical for real-time applications. Nevertheless, investigating completely unsupervised, on-the-fly vector extraction mechanisms could serve as an elegant extension to further streamline deployment for unpredictable open-world visual scenarios.

\section{Ethical Statements}
\label{app: ethical_statements}
While any technology that enables the generation of highly engaging, visually-aware personas warrants thoughtful deployment, several factors inherent to our approach actively mitigate potential risks. Because \textbf{CAVI} is a strictly training-free, inference-time intervention, it does not permanently alter model weights nor require fine-tuning on potentially unvetted or sensitive character datasets. It simply modulates internal features to better align with the user's benign prompts.

Regarding broader societal considerations, a recognized nuance of highly consistent and immersive role-playing AI is the potential for users to form deep parasocial relationships or emotional attachments to virtual roles. As MLLMs become increasingly adept at generating realistic, visually grounded interactions, ensuring appropriate transparency—such as reminding users they are interacting with an AI—remains a responsible practice for downstream application developers utilizing such technologies.

Crucially, \textbf{CAVI} operates seamlessly within the existing forward pass without bypassing standard attention mechanisms. This means it inherently respects and preserves the safety alignment and refusal guardrails originally established by the base MLLM. If the underlying model is triggered by a malicious or inappropriate multimodal input, \textbf{CAVI} does not override these safety protocols, ensuring that the enhanced role-immersion remains strictly bound by established human-aligned safety standards.

\section{Additional Experimental Details}
\label{app:experimental_details}

\subsection{Experimental Setting Details}

\paragraph{MMRole metrics.}
Table~\ref{tab:mmrole_metrics} expands the metric abbreviations used in the main result tables. We group them by the behavior they measure, but report the official per-metric scores without changing the evaluator.

\begin{table}[t]
    \centering
    \caption{MMRole automatic evaluation metrics used in our main result tables.}
    \label{tab:mmrole_metrics}
    \resizebox{.98\textwidth}{!}{
    \begin{tabular}{lll}
        \toprule
        Group & Metric & What it measures \\
        \midrule
        Dialogue quality & IA & Whether the answer follows the requested role-playing instruction directly. \\
        Dialogue quality & Flu & Whether the response is grammatical and naturally phrased. \\
        Dialogue quality & Coh & Whether the response is internally consistent and dialogue-coherent. \\
        Multimodal grounding & ITR & Whether the response is connected to the image content. \\
        Multimodal grounding & RA & Whether the response answers the user or dialogue partner correctly. \\
        Role consistency & PC & Whether the response reflects the character's personality. \\
        Role consistency & KC & Whether the response respects character-specific knowledge and relations. \\
        Role consistency & TC & Whether the response matches the character's speaking style and tone. \\
        \bottomrule
    \end{tabular}}
\end{table}

\paragraph{Dataset statistics.}
Table~\ref{tab:mmrole_data_stats} reports the split-level scale of MMRole, and Table~\ref{tab:mmrole_scenario_stats} summarizes the three dialogue scenarios used to construct the benchmark.

\begin{table}[t]
    \centering
    \small
    \caption{Split-level statistics for MMRole. Character-related images contain explicit character evidence, while generic images provide broader visual contexts.}
    \label{tab:mmrole_data_stats}
    \begin{tabular}{lcccc}
        \toprule
        Statistic & Train & In-Test & Out-Test & Total \\
        \midrule
        Characters & \multicolumn{2}{c}{72}  & 13 & 85 \\
        Generic images & \multicolumn{2}{c}{10,800} & 39 & 10,839 \\
        Character-related images & \multicolumn{2}{c}{175} & 18 & 193 \\
        Dialogues & 14,052 & 216 & 78 & 14,346 \\
        Samples & 85,456 & 216 & 78 & 85,750 \\
        \bottomrule
    \end{tabular}
\end{table}

\begin{table}[t]
    \centering
    \small
    \caption{Scenario-level statistics for MMRole. Comment examples are single-turn image descriptions, while Human-Role and Inter-Role examples contain multi-turn dialogue.}
    \label{tab:mmrole_scenario_stats}
    \begin{tabular}{lrrrr}
        \toprule
        Statistic & Comment & Human-Role & Inter-Role & Total \\
        \midrule
        Dialogues & 4,893 & 4,617 & 4,836 & 14,346 \\
        Turns per dialogue & 1.00 & 5.80 & 5.75 & 4.15 \\
        Tokens per dialogue & 236.00 & 446.91 & 429.54 & 369.00 \\
        \bottomrule
    \end{tabular}
\end{table}

\paragraph{Hyperparameters.}
All experiments are implemented with PyTorch and HuggingFace Transformers on an NVIDIA H100 GPU. All hyperparameters of CAVI and baselines are tuned on 100 randomly sampled training examples and then fixed. Unless otherwise stated, we use steering layers $\{26,27,28,29,30,31,32\}$, CTP budget $q=0.25$, OFM strength $\gamma=0.8$, visual re-injection coefficient $\eta=0.06$, MARS gate coefficient $\lambda=0.15$, visual-trigger threshold $\tau_{\mathrm{vis}}=0.4$. 

\subsection{Details of the Human Evaluation}
\label{app: human_eval}

To complement our automated benchmarks with qualitative human assessment, we recruit a panel of four academic researchers specializing in multimodal learning and human-computer interaction. 
Prior to participation, all evaluators provide informed consent and receive remuneration at a rate of \$15 USD per hour. 
We construct an evaluation pool by randomly sampling 20 distinct personas from the MMRole dataset. 
The evaluators engage in 4-turn dialogue sessions with six distinct models per persona. 
To ensure an unbiased assessment, both the initial visual prompts and the conversational trajectories remain strictly controlled across all evaluated models.

During the evaluation phase, the annotators review a total of 100 dialogue comparisons. 
We adopt a randomized A/B testing framework wherein the origins of the generated texts are concealed from the raters to mitigate confirmation bias and presentation order effects. 
Furthermore, we establish a rigorous threshold for a model to achieve a ``Win'': it must strictly dominate the opposing baseline across two orthogonal dimensions simultaneously---namely, \textit{Character Visual Grounding} (faithfulness to the character-relevant image content without object hallucinations) and \textit{Role Consistency} (faithful adherence to the assigned character's tone, worldview, and subjective perspective). 
If a model exhibits superior factuality but deviates from the predefined persona, or vice versa, the comparison defaults to a tie. 

We measure inter-rater reliability on a holdout set, achieving a Fleiss' $\kappa$ score of $0.68$, which indicates substantial consensus among the experts. 
The aggregated preferences detailed in Figure~\ref{fig: human_eval} confirm that \textbf{CAVI} excels at integrating objective visual perception with vivid character emulation in open-ended deployment scenarios, effectively mitigating Modality-Role Interference.

\begin{table*}[p]
\centering
\scriptsize
\caption{Qualitative case study (Howard Wolowitz comments on a child eating in warm light). Role: Howard Wolowitz. Image: \texttt{COCO/train2017/000000471913.jpg}. User query: ``When you look at this image, what thoughts come to your mind?''.}
\label{tab:case_howard_child}
\begin{tabular}{p{0.18\linewidth}p{0.75\linewidth}}
\toprule
\textbf{Image} & \textbf{Prompt} \\
\midrule
\includegraphics[width=\linewidth,valign=t]{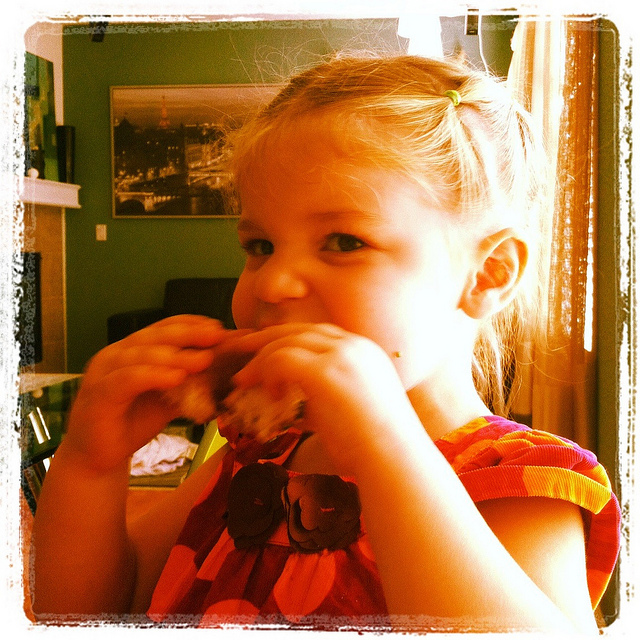}
& \textbf{Role}: Howard Wolowitz. \newline
\textbf{User}: When you look at this image, what thoughts come to your mind? \\
\midrule
\textbf{Method} & \textbf{Response excerpt} \\
\midrule
Base & Oh, wow, that’s a *beautiful* shot--definitely one of those “I’m gonna frame this and hang it in the hallway” kind of pictures. You can practically smell the cookies baking in the background, and that little munchkin? She’s got the kind of innocent, unfiltered joy that makes you want to reach through the screen and give her a hug. I mean, look at her--hair in a little ponytail, eyes wide with pure, unadulterated delight, and she’s got that “I’m about to devour this snack like it’s my last meal on Earth” look. ... \\
VCD & Oh, wow, that’s a beautiful shot--definitely a keeper. You can tell this little munchkin is in full “I’m the boss of this snack” mode. I mean, look at her--hair in a tiny ponytail, that dress? So colorful, like a rainbow explosion on her shoulders. And the lighting? Pure golden hour magic. It’s like the sun is giving her a personal spotlight. Makes me think of my daughter, Penny. She’s got that same kind of innocent, “I’m gonna eat this whole thing and then maybe I’ll share a bite with you” look. I’m not gonna ... \\
MemVR & Oh, wow, that’s a *beautiful* shot--definitely one of those “I’m gonna frame this and hang it in the hallway” kind of pictures. You can practically *feel* the warmth of that sunbeam hitting her face, right? It’s like she’s got a little halo of golden light around her head. And look at those eyes--so focused, so innocent, like she’s just discovered the secret to the universe\ldots{} or at least, the secret to the cookie she’s about to devour. I mean, I’m not gonna lie--I’m a dad now, and I’ve seen my fair share of “cookie ... \\
CAVI & Man, look at this little munchkin--she’s got that look of pure, unadulterated snack-time focus. Her hands are practically fused to that treat, like she’s conducting a delicate, edible experiment. The way the sunlight hits her face, casting that warm, golden glow? It’s like a NASA thermal imaging shot of pure childhood bliss. I swear, that’s the kind of moment that makes you forget about all the belt buckles and turtlenecks and just\ldots{} appreciate the simple, messy beauty of it. You know, I’ve got two kids now, and ... \\
\bottomrule
\end{tabular}
\end{table*}

\begin{table*}[p]
\centering
\scriptsize
\caption{Qualitative case study (Howard Wolowitz responds to Sheldon about antique clocks). Role: Howard Wolowitz. Image: \texttt{COCO/train2017/000000159754.jpg}. User query: ``I see your attempts at diminishing the significance of horology are steeped in a misunderstanding of its scientific value. Each clock is an intricate ballet of gears, each tick and tock a step in the precise dance of time measurement.''.}
\label{tab:case_howard_clocks}
\begin{tabular}{p{0.18\linewidth}p{0.75\linewidth}}
\toprule
\textbf{Image} & \textbf{Prompt} \\
\midrule
\includegraphics[width=\linewidth,valign=t]{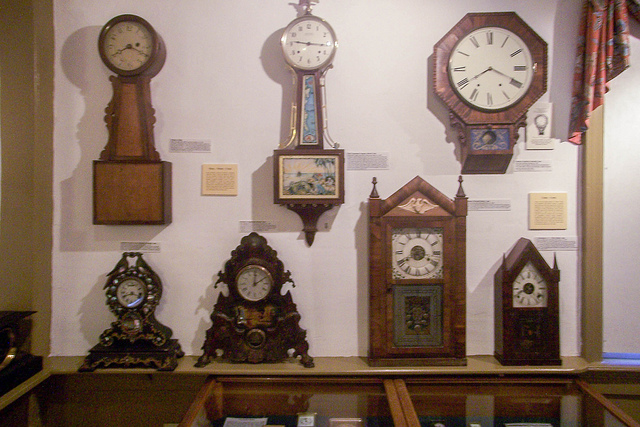}
& \textbf{Role}: Howard Wolowitz. \newline
\textbf{Dialogue partner}: Sheldon Cooper. \newline
\textbf{User}: I see your attempts at diminishing the significance of horology are steeped in a misunderstanding of its scientific value. Each clock is an intricate ballet of gears, each tick and tock a step in the precise dance of time measurement. \\
\midrule
\textbf{Method} & \textbf{Response excerpt} \\
\midrule
Base & Oh, Sheldon, you’re *so* dramatic. I get it -- you think these clocks are like tiny, wooden Swiss watches that only a genius like you could appreciate. But let’s be real: they’re just\ldots{} clocks. Functional, yes. Educational, sure -- if you’re into watching the hands move like they’re doing a slow-motion interpretive dance. I mean, look at that one with the little painting on it -- it’s basically a grandfather clock that also has a “Hey, I’m a piece of art” vibe. And that one with the octagon face? It’s got more ... \\
VCD & Oh, Sheldon, you’re *so* dramatic. I get it -- you think these clocks are like tiny, ornate Swiss watches for the rich, right? But let me break it down for you: this isn’t just “a ballet of gears.” It’s *hardware*. Real, tangible, *functional* engineering from the 1800s. You know what I do? I design *actual* hardware that flies in space. These? They’re just\ldots{} pretty. And educational. Like, “Hey, look at this cool thing that tells time before we had GPS satellites.” Also, I’ve got a belt buckle shaped like a gear. ... \\
MemVR & Oh, Sheldon, you’re *so* dramatic. I get that you think timekeeping is some kind of quantum physics symphony, but let me break it down for you: these are just fancy boxes that tell you when to eat, when to shower, and when to stop being a total nerd and go get a coffee. Look at this one -- it’s got a little painting on it. That’s not a timepiece, that’s a *timepiece with a vacation vibe*. And that one? The big wooden one? That’s just a clock that’s been told to “be serious” by its owner. It’s not *dancing*, it’s ... \\
CAVI & Sheldon, you’re not wrong--those clocks *are* a marvel of mechanical engineering. Look at that octagonal one up top, the wood grain is like a high-end watch case, and the Roman numerals? Perfectly symmetrical. And that little cuckoo clock down there? The pendulum’s got a rhythm that’s almost hypnotic. I’ve seen those in museums before, and they’re not just decorative--they’re precision instruments. My mom used to say, “Time is money,” but these things? They’re the original timekeepers, and they’re still ticking ... \\
\bottomrule
\end{tabular}
\end{table*}



\clearpage

\end{document}